\PassOptionsToPackage{dvipsnames,table}{xcolor}

\documentclass[sigconf]{acmart}
\AtBeginDocument{%
  }

\usepackage{enumitem}
\usepackage{multirow}
\usepackage{array}
\usepackage{graphicx}
\usepackage{array}
\usepackage{xcolor}
\usepackage{booktabs}
\usepackage{multirow}
\usepackage{tabularx}
\usepackage{url}
\usepackage{tikz}
\usepackage{orcidlink}
\usepackage{cleveref}

\usepackage{algorithm}
\usepackage{algpseudocode}
\algrenewcommand\algorithmicrequire{\textbf{Input:}}
\algrenewcommand\algorithmicensure{\textbf{Output:}}
\algrenewcommand\algorithmiccomment[1]{\hfill$\triangleright$~#1}

\renewcommand\footnotetextcopyrightpermission[1]{}
\makeatletter\renewcommand*{\@fnsymbol}[1]{\hspace*{2.4em}*}\def\@authornotemark{}\makeatother

\setcopyright{acmlicensed}
\copyrightyear{2018}
\acmYear{2018}
\acmDOI{XXXXXXX.XXXXXXX}
\acmConference[Conference acronym 'XX]{Make sure to enter the correct
  conference title from your rights confirmation email}{June 03--05,
  2018}{Woodstock, NY}
\acmISBN{978-1-4503-XXXX-X/2018/06}

\acmSubmissionID{3310}

\newcommand{\blueupsmall}[1]{$_{\color{RoyalBlue}\uparrow \text{#1}}$}

\newcommand{\blueup}[1]{\textcolor{RoyalBlue}{$\uparrow$ #1}}
\newcommand{\reddown}[1]{\textcolor{OrangeRed}{$\downarrow$ #1}}

\makeatletter
\newcommand{\thickhline}{%
  \noalign{\ifnum0=`}\fi\hrule height 1pt\relax
  \futurelet\reserved@a\@xhline
}
\makeatother

\begin{document}
\pagestyle{plain}

\title{CAPT: Class-Aware Prompt Tuning for Federated Long-Tailed Learning with Vision-Language Model}

\author{Shihao Hou$^{1}$\hspace{20pt}Chikai Shang$^{1}$\hspace{20pt}Jiacheng Yang$^{1}$\hspace{20pt}Xinyi Shang$^{2}$\\
Shreyank N. Gowda$^{3}$\hspace{20pt}Yan Yan$^{1}$\hspace{20pt}Yang Lu$^{1*}$}
\authornote{Corresponding author: Yang Lu (luyang@xmu.edu.cn).\\\hspace*{2.25em}Accepted by ACM MM 2026.}
\affiliation[obeypunctuation=true]{%
  \institution{\small
    $^{1}$Key Laboratory of Multimedia Trusted Perception and Efficient Computing, Ministry of Education of China, Xiamen University, Xiamen, China\\
    $^{2}$Department of Statistical Science, University College London, London, United Kingdom\\
    $^{3}$School of Computer Science, University of Nottingham, Nottingham, United Kingdom\\
    {\ttfamily houshihao@stu.xmu.edu.cn\hspace{20pt}luyang@xmu.edu.cn}}
  \city{}\country{}}

\renewcommand{\shortauthors}{Hou et al.}

\begin{abstract}

Effectively addressing the joint effects of non-IID data and long-tailed distribution remains a key challenge in federated learning. While fine-tuning vision-language models (VLMs) like CLIP has shown promise in tackling non-IID scenarios, it often degrades tail-class performance when both heterogeneity and long-tailed distributions are present, even leading to overall performance falling below the zero-shot baseline. 
To overcome this, we propose Class-Aware Prompt Tuning (CAPT), a novel framework leveraging pre-trained VLMs to effectively learn from non-IID and long-tailed data.
Specifically, CAPT introduces a dual-prompt mechanism that synergizes class-aware and general prompts, enabling the model to preserve class-specific information while capturing global trends. Furthermore, to guide the prompts in learning targeted knowledge, we propose a heterogeneity-aware clustering strategy with two key components: similarity-based clusters aggregate class-aware prompts to reinforce tail-specific knowledge, while heterogeneity-based clusters aggregate general prompts to balance diverse representations and promote robust generalization. Extensive experiments demonstrate that CAPT significantly outperforms state-of-the-art methods, e.g., by improving overall accuracy by 6.1\% on ImageNet-LT and 8.8\% on Places-LT over baselines, while also substantially enhancing tail-class performance across other benchmarks.

\end{abstract}





\maketitle

\section{Introduction}
\label{sec:Introduction}
Federated learning (FL) enables multiple parties to collaboratively train a shared model without exposing private data, providing a privacy-preserving solution to leverage distributed data for building a robust global model~\cite{FedAvg}. However, real-world FL faces substantial challenges since it often encounters the coexistence of non-IID data and long-tailed distributions. This combination fragments the already scarce tail-class samples across heterogeneous clients, making it difficult to learn effective representations for tail classes~\cite{Creff}. Such scenarios typically result in a global model that converges slowly and remains heavily biased towards head classes.

\begin{figure}[!t]
  \centering
  \includegraphics[width=0.48\linewidth]{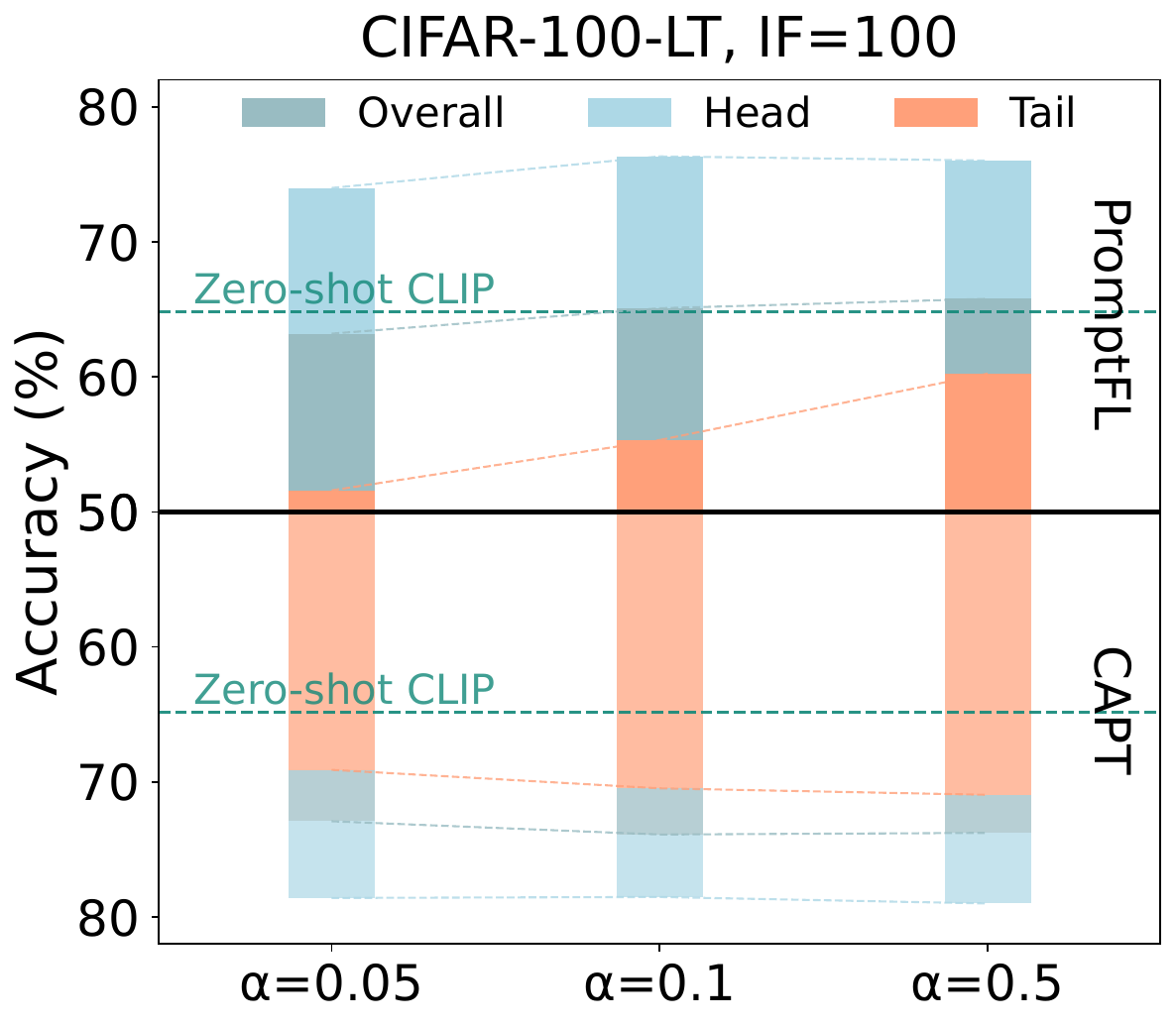}\hfil
  \includegraphics[width=0.48\linewidth]{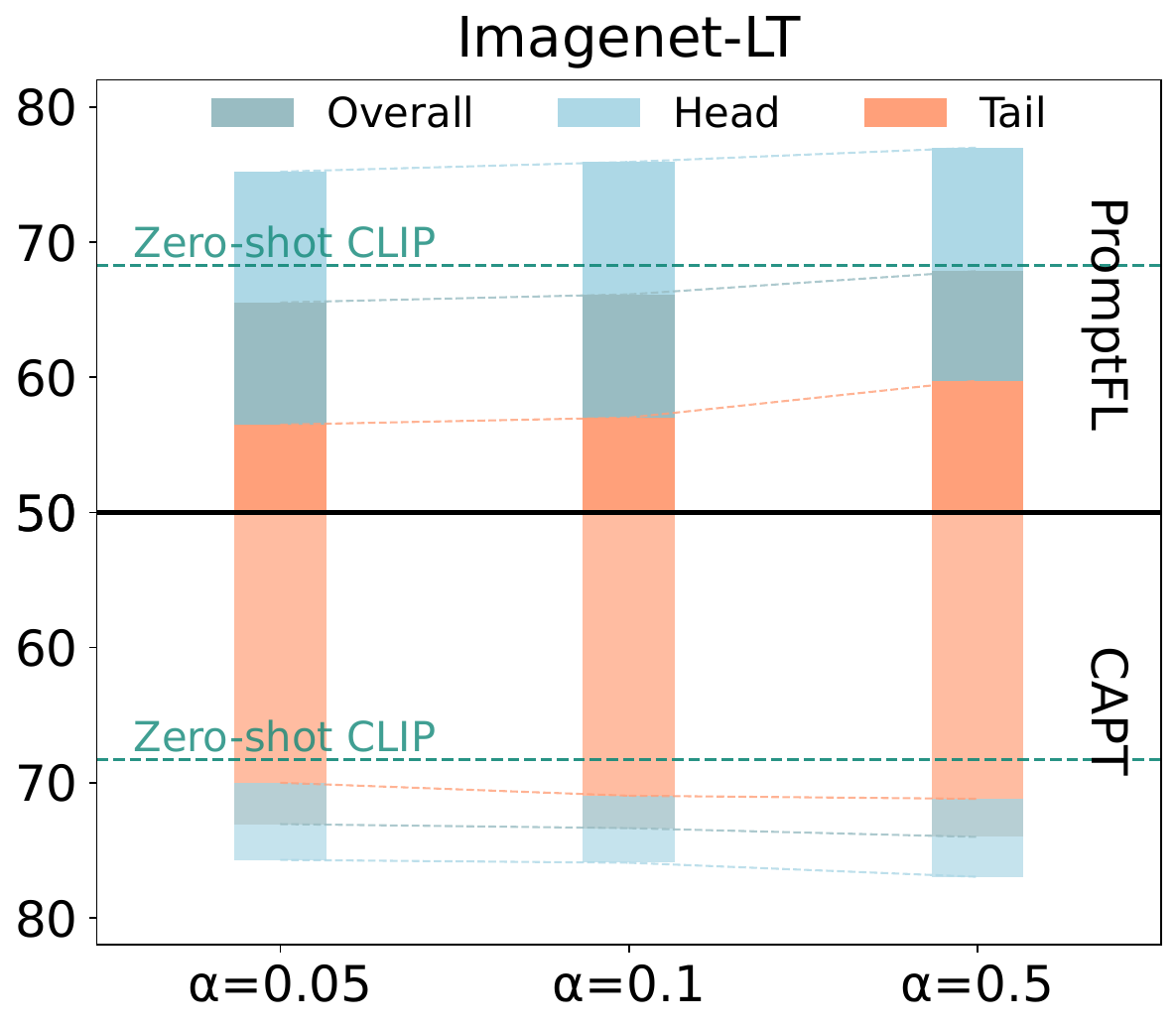}
  \caption{Performance comparison between PromptFL (upper part) and CAPT (lower part). As client heterogeneity increases ($\alpha$ decreases from 0.5 to 0.05), PromptFL shows widening head-tail gaps and declining overall accuracy below the zero-shot CLIP baseline (dashed line) due to poor tail performance. CAPT mitigates this, reducing gaps while sustaining superior overall accuracy.}
  \Description{Two side-by-side plots compare PromptFL and CAPT under different levels of client heterogeneity on CIFAR-100 and ImageNet. As alpha decreases, PromptFL shows a larger head-tail performance gap and lower overall accuracy, in some cases dropping below the zero-shot CLIP baseline marked by a dashed line. CAPT reduces the head-tail gap and maintains higher overall accuracy across heterogeneity levels.}
  \label{fig:motivation}
\end{figure}

Existing methods addressing federated long-tailed learning (Fed-LT) generally fall into three categories: loss-centric strategies, which re-weight local training loss functions to mitigate class imbalance~\cite{Fed-Focal_Loss}; representation-centric strategies, which enhance representation learning by decoupling model components~\cite{Creff}; and coordination-centric strategies, which use global information such as class prototypes or proxies to guide and align client models~\cite{RUCR,GBME}. Despite reported improvements in overall accuracy, these approaches often obscure a key issue: tail-class performance remains poor, leaving the significant head-tail gap unresolved~\cite{Fed-Grab}. More importantly, these methods do not leverage the powerful foundation models. Federated prompt learning (FPL) offers a novel paradigm by leveraging prompts to efficiently harness foundation models in federated settings~\cite{PromptFL}. By learning and aggregating prompts alone, FPL effectively adapts to non-IID data while utilizing the knowledge of pre-trained models. Nevertheless, the potential of prompt tuning for addressing Fed-LT challenges remains unexplored.

To bridge this gap, we first examine the behavior of prompt-based learning~\cite{PromptFL} in Fed-LT. We conduct a preliminary analysis under varying degrees of client heterogeneity, as illustrated in Fig.~\ref{fig:motivation}. Our results show that as client heterogeneity increases, performance on tail classes degrades significantly, revealing that naively applying prompt tuning to Fed-LT can further widen the accuracy gap between head and tail classes.

We investigate this phenomenon from two perspectives, and uncover a shared limitation: the gradients of tail classes are systematically overwhelmed by dominant head-class gradients. Inspired by this finding, we propose \textbf{C}lass-\textbf{A}ware \textbf{P}rompt \textbf{T}uning (CAPT), a framework that preserves tail-specific knowledge while learning robust shared features through dual-prompt decoupling. Specifically, to capture fine-grained class details, especially for the tail, we decouple a class-aware prompt whose updates are aggregated within similarity-based client clusters, enabling collaborative reinforcement of tail-specific knowledge rather than allowing tail details to be diluted. Additionally, we employ a general prompt that learns domain-invariant features aggregated via heterogeneity-based clustering among complementary clients, thereby promoting diverse feature extraction across heterogeneous clients. The main contributions of this work are: 

\begin{itemize}[leftmargin=5mm, itemsep=0mm, topsep=0.5mm]
\item We reveal that existing prompt-tuning methods degrade severely on tail classes in Fed-LT, motivating new solutions that jointly address non-IID and long-tailed data.

\item We propose a dual-prompt mechanism in CAPT that decouples class-aware prompts for per-class representations from general prompts for shared features, preserving tail knowledge while capturing robust global trends.
\item We propose a heterogeneity-aware clustering in CAPT that groups similar clients for class-aware optimization and complementary clients for general prompt refinement, enabling efficient knowledge aggregation.
\item Our approach achieves state-of-the-art performance on multiple long-tailed datasets, significantly improving tail-class accuracy (e.g., by up to 11.4\% on ImageNet-LT) while also enhancing overall accuracy.
\end{itemize}

\section{Related Work}
\subsection{Federated Long-Tailed Learning}
Federated long-tailed learning (Fed-LT) tackles the compounded challenge where long-tailed class priors fragment already scarce tail samples across heterogeneous clients, yielding global models heavily biased toward head classes. Prior work has explored several directions to mitigate this bias. Loss rebalancing methods~\cite{Fed-Focal_Loss,Ratio_loss} re-weight local objectives to amplify tail contributions, but often rely on unrealistic assumptions about global distributions. Feature decoupling approaches~\cite{FEDIC,Creff} separate representation learning from classifier training to reduce head-class dominance, yet remain constrained by the quality of federated features. More recent coordination-based methods~\cite{Fed-Grab,RedGrape,Fedloge,CLIP2FL,GBME} share statistics or prototypes across clients to improve class-wise alignment, though at the cost of substantial communication overhead~\cite{FLsurvey}. Despite these advances, a common limitation persists: all these methods train models from scratch without exploiting the generalizable representations of pre-trained foundation models, making it inherently difficult to learn robust tail-class features from sparse and fragmented data. In this work, we instead build upon pre-trained VLMs and introduce specialized prompt tuning to address both heterogeneity and class imbalance more efficiently.

\subsection{Prompt Learning in VLMs}
Prompt learning has emerged as an effective paradigm for adapting large VLMs like CLIP~\cite{clip} to downstream tasks without sacrificing generalization. Early efforts focus on the language branch: CoOp~\cite{coop} introduces learnable context tokens for text prompts, and CoCoOp~\cite{cocoop} further conditions prompts on input images to improve generalization to unseen classes. MaPLe~\cite{maple} extends this idea to both the visual and textual branches, enabling cross-modal adaptation. Beyond these, various works have explored complementary directions such as knowledge-guided optimization~\cite{LASP, KgCoOp}, multi-label recognition~\cite{dualcoop}, and unsupervised prompt tuning~\cite{UPL}, demonstrating the versatility of prompt-based VLM adaptation.

Recent efforts have brought prompt learning into federated settings to handle data heterogeneity. PromptFL~\cite{PromptFL} replaces full model training with distributed prompt optimization, substantially reducing communication cost. Subsequent methods further improve this paradigm: FedOTP~\cite{fedotp} and FedPGP~\cite{fedpgp} introduce mechanisms for prompt sharing and personalization, while pFedPrompt~\cite{pFedPrompt} targets multimodal consensus and FedTPG~\cite{fedtpg} generates prompts from text descriptions. However, these methods primarily address heterogeneity or personalization, without considering the long-tailed class distributions common in real-world FL. As a result, they suffer from inconsistent updates and significant tail-class degradation under class imbalance. Our CAPT framework addresses this gap by combining class-aware prompts with heterogeneity-aware clustering, explicitly tackling federated long-tailed learning for robust tail-class performance.

\section{Theoretical Analysis}
\label{sec:theoretical_framework}

We provide a theoretical framework to explain the empirical failures of vanilla federated prompt tuning (e.g., PromptFL) under non-IID and long-tailed distributions. As shown in Fig.~\ref{fig:motivation}, tail-class accuracy degrades sharply with increasing client heterogeneity, widening the head-tail gap and often falling below zero-shot baselines. Our analysis decomposes the gradient error into bias and variance components relative to an ideal balanced objective, formally explaining this phenomenon.


 

\subsection{Preliminaries and Problem Setup}
\label{sec:setup_notation}

We consider a federated learning system with $N$ clients and a central server, where each client $k$ has a local dataset $D_k = \{(\mathbf{x}_i, y_i)\}_{i=1}^{|D_k|}$ with $\mathbf{x}_i \in \mathbb{R}^d$ and $y_i \in \{1, \dots, C\}$. The global dataset $D = \bigcup_{k=1}^{N} D_k$ follows a long-tailed distribution, with classes ordered as $n_1 \geq n_2 \geq \dots \geq n_C$, where $n_c$ is the sample count for class $c$. We refer to classes with larger $n_c$ as \emph{head} classes and those with smaller $n_c$ as \emph{tail} classes. The global class prior is denoted by $q \in \Delta^{C-1}$, with imbalance ratio $\rho = \max_c q_c / \min_c q_c$. Each client $k$ holds $n_k$ samples with empirical class proportions $\hat{p}_k$. To model client heterogeneity, we assume a client-specific class prior $p_k \sim \mathrm{Dir}(\alpha q)$, where a smaller $\alpha$ indicates stronger heterogeneity.


Our work builds upon CLIP, which maps images and texts into a shared embedding space via an image encoder $f(\cdot)$ and a text encoder $g(\cdot)$. To adapt CLIP to downstream tasks without fine-tuning the entire model, we employ prompt learning~\cite{coop}, optimizing learnable prompt $\mathbf{P} = [\mathbf{V}_1, \mathbf{V}_2, \dots, \mathbf{V}_M] \in \mathbb{R}^{M \times d}$, where each $\mathbf{V}_j$ is a context token initialized with pre-trained word embeddings.

For theoretical analysis, we treat $\mathbf{P}$ as a vectorized parameter in $\mathbb{R}^m$, and let $\mu_c(\mathbf{P}) \in \mathbb{R}^m$ denote the class-conditional mean gradient for class $c$, assumed to be client-invariant. Each client $k$ optimizes a local objective with gradient $\nabla F_k(\mathbf{P}) = \sum_c \hat{p}_{k,c}\, \mu_c(\mathbf{P})$, and the server aggregates these to approximate the population gradient $\nabla F(\mathbf{P}) = \sum_c q_c\, \mu_c(\mathbf{P})$.
In long-tailed settings, however, the desired target is the balanced objective $F_{\mathrm{bal}}(\mathbf{P}) = \frac{1}{C}\sum_c F_c(\mathbf{P})$, whose gradient $\nabla F_{\mathrm{bal}}(\mathbf{P}) = \frac{1}{C}\sum_c \mu_c(\mathbf{P})$ weights all classes equally.


\subsection{Gradient Error Decomposition}
\label{sec:gradient_error}

\noindent\textbf{Prior-Induced Bias.}
Since the server's aggregated gradient is an unbiased estimator of the population gradient ($\mathbb{E}[\widehat{g}] = \nabla F$), the bias is entirely determined by the deterministic gap between $\nabla F$ and $\nabla F_{\mathrm{bal}}$. The following lemma quantifies this gap.
 
\begin{lemma}[Prior-Induced Gradient Bias]
\label{lem:prior-bias}
Let $u$ denote the uniform prior, and let
$M(\mathbf{P}) = \mu(\mathbf{P})^\top \mu(\mathbf{P})$
be the Gram matrix of the class-conditional mean gradients. The squared distance between the population and balanced gradients is:
\begin{equation}
    \|\nabla F(\mathbf{P})-\nabla F_{\mathrm{bal}}(\mathbf{P})\|^2=(q-u)^\top M(\mathbf{P})\,(q-u).
\end{equation}
\end{lemma}
\noindent Optimizing the standard objective $F$ is not a neutral act; it actively pulls the model away from the desired balanced solution $F_{\mathrm{bal}}$ with a force proportional to the imbalance ratio $\rho$.

\noindent\textbf{Between-Client Variance.}
Beyond this deterministic bias, the federated setting introduces significant stochasticity from client heterogeneity. The variance in clients' empirical class proportions $\hat{p}_k$ directly translates into variance in their local gradients.
 
\begin{lemma}[Expected Between-Client Gradient Variance]
\label{lem:between}
Under the Dirichlet-Multinomial model, the expected squared distance between a client's local gradient and the global mean gradient is:
\begin{equation}
\mathbb{E}\big[\|\nabla F_k-\nabla F\|^2\big]
= \frac{\alpha+n_k}{n_k(\alpha+1)}\cdot
   \mathrm{Var}_{c\sim q}\!\Big[\mu_c(\mathbf{P})\Big],
\end{equation}
where $\mathrm{Var}_{c\sim q}[\mu_c(\mathbf P)]
= \sum_{c=1}^C q_c\|\mu_c(\mathbf P)\|^2 - \big\|\sum_{c=1}^C q_c\,\mu_c(\mathbf P)\big\|^2$.
\end{lemma}
\noindent This variance neatly factors into a statistical component (driven by non-IID level $\alpha$ and sample size $n_k$) and a geometric component (the $q$-weighted spread of class gradients). Both are exacerbated by long-tailed distributions, explaining why heterogeneity is particularly damaging in such settings. When only a subset of $K$ clients participates per round, this variance is further modulated by a finite-population correction factor $\frac{1-K/N}{K}$.


\noindent\textbf{Main Result and Design Principles.}
Synthesizing the above results, we arrive at the main theorem that characterizes the total per-round error of the vanilla federated prompt tuning process.
 
\begin{theorem}[MSE w.r.t.\ the Balanced Gradient]
\label{thm:main_mse}
The mean squared error of the server's aggregated gradient $\widehat{g}$ satisfies:
\begin{equation}
\begin{aligned}
\mathbb{E}\big[\|\widehat{g} - \nabla F_{\mathrm{bal}}\|^2\big]
&\le \underbrace{(q - u)^\top M(\mathbf{P}) (q - u)}_{\text{Prior-Induced Bias}} \\
&\quad + \underbrace{\tfrac{1 - K/N}{K}\,\Phi(\alpha, \{n_k\}; \mathbf P)}_{\text{Between-Client Variance}},
\end{aligned}
\end{equation}
where $\Phi(\alpha, \{n_k\}; \mathbf P) = \frac{1}{N}\sum_{k=1}^N \frac{\alpha+n_k}{n_k(\alpha+1)}\cdot \operatorname{Var}_{c\sim q}\!\big[\mu_c(\mathbf P)\big]$.
\end{theorem}
 Under standard $L$-smoothness assumptions, this MSE directly governs the convergence behavior of the prompt update rule: the two error components act as a persistent floor that prevents convergence to a stationary point of $F_{\mathrm{bal}}$, regardless of how many communication rounds are performed. This analysis motivates two design principles for effective federated prompt tuning under long-tailed non-IID data:
\begin{enumerate}[leftmargin=5mm, itemsep=0mm, topsep=0.5mm]
    \item \textbf{Debiased Objective Alignment:} Mitigate the prior-induced bias to prevent head-class dominance.
    \item \textbf{Targeted Aggregation:} Reduce between-client variance through selective client grouping, mitigating heterogeneity's impact on tail classes.
\end{enumerate}

\begin{figure}[!t]
\centering
 \includegraphics[width=1.0\columnwidth]{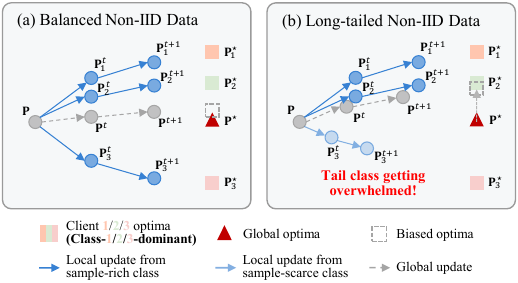}
 \caption{Effects of non-IID and long-tailed data. (a) Non-IID data causes clients to be dominated by a few classes, resulting in divergent local updates that deviate from the optimal global update. (b) When combined with long-tailed patterns, clients with rare classes are outnumbered by those with abundant classes, biasing aggregated global updates toward heads and impairing tail-class representation.}
 \label{fig:motivat}
\end{figure}

\section{Proposed Method}




\begin{figure*}[!t]
  \centering
  \includegraphics[width=0.95\textwidth]{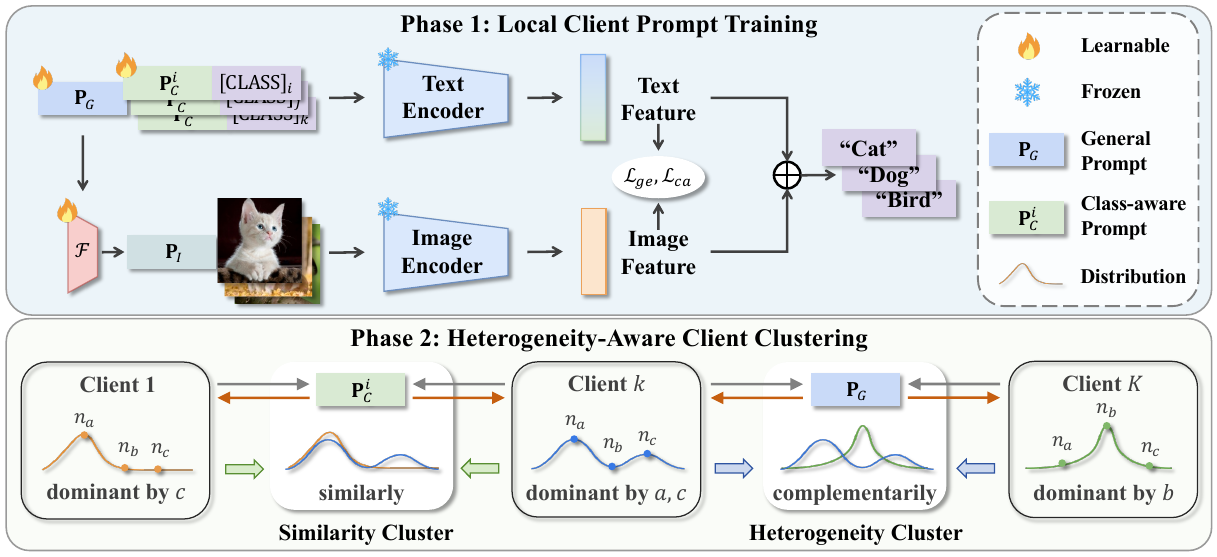}
  \caption{Overview of the proposed CAPT framework. Phase 1 illustrates local client prompt training, where a general prompt \(P_G\) and class-aware prompts \(P_C^i\) are optimized using a frozen image encoder and text encoder, with losses \(\mathcal{L}_{ge}\) and \(\mathcal{L}_{ca}\) computed from text and image features for classification. Phase 2 depicts heterogeneity-aware client clustering, showing similarity-based clustering that groups clients with dominant classes (e.g., by \(c\)) for aggregating class-aware prompts, and heterogeneity-based clustering that groups complementary clients (e.g., dominant by \(a,c\) and by \(b\)) for aggregating general prompts.}
  \label{fig:method}
\end{figure*}

\subsection{Motivation}
\label{sec:motivation}

As discussed in Fig.~\ref{fig:motivation}, the head-tail performance gap widens with increasing client heterogeneity. We attribute this to the class-dominated property of non-IID data, where in Fed-LT, the tail-class-dominant gradients are significantly outnumbered and overwhelmed by those from head classes, as illustrated in Fig.~\ref{fig:motivat}. Our theoretical analysis in Sec.~\ref{sec:theoretical_framework} formalizes this intuition: the prior-induced bias (Lemma~\ref{lem:prior-bias}) shifts optimization away from balanced solutions, while the between-client variance (Lemma~\ref{lem:between}) escalates with lower $\alpha$, stalling tail-class progress. Per Theorem~\ref{thm:main_mse}, these two sources of error create a persistent floor that undermines tail learning. The key limitation is clear: \textit{the learning process for tail classes is undermined as their gradients are overwhelmed by head-class updates.}

A natural remedy is re-weighting the prompt updates to amplify tail gradients and curb the prior-induced bias. However, this risks overfitting scarce tail samples by learning spurious correlations~\cite{cao2019learning,wang2020long}, and fails to decouple fine-grained class details from domain-invariant features, leaving the geometric spread $\mathrm{Var}_{c\sim q}[\mu_c(\mathbf{P})]$ that inflates between-client variance unaddressed. To tackle both error sources identified in Theorem~\ref{thm:main_mse}, we propose a two-level decoupling strategy: (i) separating class-aware prompts for per-class representations from general prompts for shared features (Sec.~\ref{sec:dual_prompt}); and (ii) aggregating them via distinct clustering schemes, where similar clients reinforce class-aware prompts and complementary clients refine general prompts (Sec.~\ref{sec:client_clustering}). This design preserves tail-class gradients while promoting robust global convergence.

\subsection{Dual-Prompt Learning Mechanism}
\label{sec:dual_prompt}
To mitigate gradient domination, we decouple prompt tuning into two complementary components: a class-aware prompt, safeguarding fine-grained tail details, and a general prompt, establishing domain-invariant patterns. We further integrate a vision–language alignment mechanism to bridge the semantic gap between textual and visual feature spaces, mitigating alignment distortions caused by tail-class gradients.

\noindent \textbf{Class-Aware Prompt.} 
To fairly learn both head and tail classes, we introduce class-aware prompts $\textbf{P}_c \!\in\! \mathbb{R}^{C \times T \times d}$ to capture class-specific characteristics, allowing for more detailed representations of both head and tail classes. Where $T$ denotes the number of learnable tokens per class. For class $i$, its corresponding prompt is represented as $\textbf{P}_c^i \!= \! [\mathbf{V}_i^{1}, \mathbf{V}_i^{2}, ..., \mathbf{V}_i^{T}] \!\in\! \mathbb{R}^{T \times d}$. During training, only the prompt corresponding to the true class of each sample is updated and optimized with the following loss function:
\begin{equation}
    \mathcal{L}_{\text{ca}} = -\frac{1}{N}\sum_{i=1}^N \log \frac{\hat{\pi}_{y_i}\, \exp(\text{cos}(\mathbf{t}_c^{y_i}, \mathbf{z}_i) / \tau)}{\sum_{j=1}^C \hat{\pi}_j\, \exp(\text{cos}(\mathbf{t}_c^j, \mathbf{z}_i) / \tau)},
\label{eq:L_cls}
\end{equation}
where \(y_i\) is the true class, \(\mathbf{t}_c^j = g([\textbf{P}_c^j, \text{CLASS}_j])\) is the text feature from the class-aware prompt and class token. Due to privacy constraints in FL, the global label distribution is unavailable, we estimate the global class prior $\hat{\pi}$ from a one-shot accumulated-gradient proxy, inspired by prior work~\cite{GBME}. On each client $k$ with local batches $b$, we first compute the local proxy for class $c$:
\begin{align}
\Lambda_c^{k} &= \sum_{t=1}^{T} \sum_{b=1}^{B_k} \Big\lVert \nabla_{\mathbf{V}_c^{t}} \mathcal{L}_{\text{ca}}(\mathcal{B}_{k,b}) \Big\rVert_1. \label{eq:local_proxy}
\end{align}
These proxies are uploaded once in the first round, and the server aggregates and normalizes them:
\begin{equation*}
\hat{\pi}_c = \frac{\Pi_c}{\sum_{j=1}^{C} \Pi_j}, \quad \text{where} \quad \Pi_c = \sum_{k=1}^{K} \frac{n_k}{n} \, \mathbb{I}(\Lambda_c^{k} > 0) \, \Lambda_c^{k}.
\label{eq:normalized_prior}
\end{equation*}
where $n_k$ is the sample size of client $k$ and $n = \sum_k n_k$. While effective for preserving unique class details, relying solely on this may limit shared learning, risking further gradient disagreement in heterogeneous settings.

\noindent \textbf{General Prompt.}
To counter this and provide a stable foundation of shared knowledge, we introduce the general prompt $\textbf{P}_g$ to capture universal, domain-invariant features. This prompt is optimized on all samples using the following contrastive loss, ensuring it learns common knowledge from all classes in the local dataset:
\begin{equation}
    \mathcal{L}_{\text{ge}} = -\frac{1}{N}\sum_{i=1}^N \log \frac{\exp(\text{cos}(\mathbf{t}_g^i, \mathbf{z}_i) / \tau)}{\sum_{j=1}^C \exp(\text{cos}(\mathbf{t}_g^j, \mathbf{z}_j) / \tau)},
\label{eq:L_gen}
\end{equation}
where \(\mathbf{z}_i \!= \!f(\mathbf{x}_i)\) is the image feature, \(\mathbf{t}_g^i \!= \! g([\textbf{P}_g, \text{CLASS}_i])\) is the text feature, and \(\tau\) is the temperature. This stabilizes global updates, complementing the class-aware isolation to prevent overall gradient dominance.

\noindent \textbf{Vision-Language Alignment.}
While our dual-prompt mechanism effectively learns general and class-specific features in the text domain, aligning these textual representations with their visual counterparts remains a challenge. Specifically, gradients from tail samples can distort this cross-modal alignment, exacerbating the head-class bias. To mitigate this, we introduce a mapping function $\mathcal{F}\!:\! \mathbb{R}^d \!\rightarrow \!\mathbb{R}^v$ inspired by MaPLe~\cite{maple}, which aligns textual prompts with the visual feature space:
\begin{equation}
    \mathbf{z}^{\prime}_i = f([\mathbf{x}_i, \tilde{\textbf{P}}_{g}]), \quad \text{where} \quad \tilde{\textbf{P}}_{g} = \mathcal{F}(\textbf{P}_g).
\end{equation}
For each sample, we construct an integrated prompt representation as $\textbf{P}_i \!= \![\textbf{P}_g, \textbf{P}_c^{y_i}, \text{CLASS}_i]$, processed via the text encoder: \(\mathbf{t}_i = g(\textbf{P}_i)\). This ensures tail gradients propagate undistorted, reinforcing the decoupling's effectiveness. During inference, we concatenate the general prompt, the prompt corresponding to the predicted class, and the class token to form the final prompt embedding. This unified representation captures both shared and class-specific features.

The total local loss balances these two objectives: \(\mathcal{L} \!=\! \mathcal{L}_\text{ca} + \lambda\mathcal{L}_\text{ge}\), where \(\lambda\) weights generalization against specificity. This joint optimization protects tail gradients while building robust shared knowledge, producing class-aware and general prompts that are aggregated through heterogeneity-aware clustering.

\noindent \textbf{Inference Process}
\label{sec:inference}
Given an image $\mathbf{x}$, we first obtain the mapped general prompt $\tilde{\textbf{P}}_{g}=\mathcal{F}(\textbf{P}_{g})$ and compute the visual feature$\mathbf{z}=f([\mathbf{x}, \tilde{\textbf{P}}_{g}])$.
For each class $j\in[C]$, we construct the text feature using both the general and class-aware prompts:
\begin{equation}
\mathbf{t}_j = g([\textbf{P}_g, \textbf{P}_c^{j}, \text{CLASS}_j]).
\end{equation}
We then score classes by cosine similarity and predict via
\begin{equation}
\hat{y}=\arg\max_{j\in[C]} \cos\!\big(\mathbf{t}_j,\mathbf{z}\big).
\end{equation}
At inference, we use the aggregated prompts after server-side clustering, ensuring the final decision leverages both general knowledge $\textbf{P}_g$ and class-specific features $\textbf{P}_c$.

\subsection{Heterogeneity-Aware Client Clustering}
\label{sec:client_clustering}

While the dual-prompt mechanism effectively decouples feature learning at the local level, naive aggregation methods like FedAvg risk reintroducing head-biased dilution. To address this, we propose a heterogeneity-aware client clustering strategy. As shown in Phase 2 of Fig.~\ref{fig:method}, this strategy comprises two complementary methods: similarity-based clustering to safeguard tail-specific gradients in class-aware prompts, and heterogeneity-based clustering to balance diverse representations in general prompts.

\noindent \textbf{Similarity-based Clustering.}
To shield tail-class gradients in the class-aware prompt from dilution during aggregation, we group clients with similar distributions, enabling focused sharing of scarce tail signals without interference from head-dominant data.

Similarity between normalized class distributions \(\mathbf{d}_i\) and \(\mathbf{d}_j\) is quantified using Jensen-Shannon (JS) divergence:
\begin{equation}
    \text{JS}(\mathbf{d}_i || \mathbf{d}_j) = \frac{1}{2}D_{\text{KL}}(\mathbf{d}_i || \mathbf{d}_M) + \frac{1}{2}D_{\text{KL}}(\mathbf{d}_j || \mathbf{d}_M),
\end{equation}
where $\mathbf{d}_i$ and $\mathbf{d}_j$ denote the class distribution vectors for the respective clients (estimated by normalizing per-class gradient magnitudes from Eq.~\eqref{eq:local_proxy}), $\mathbf{d}_M \!=\! \tfrac{1}{2}(\mathbf{d}_i+\mathbf{d}_j)$, and $D_{\text{KL}}$ is the Kullback–Leibler divergence. A similarity matrix $\mathbf{S}$ is constructed, and clusters are formed using the K-means algorithm. This enables intra-cluster aggregation of class-aware prompts among clients with similar tail-class distributions, reducing inter-client gradient variance and strengthening tail features without head-class domination.

\noindent \textbf{Heterogeneity-based Clustering.}
Conversely, the general prompt focuses on domain-invariant features, similarity-only clustering would constrain diversity and perpetuate head biases. To address this, we group clients with complementary distributions, fostering balanced feature synthesis and robust tail contributions.

Complementarity between clients \(i\) and \(j\) is quantified as:
\begin{equation}
    \text{Comp}(\mathbf{d}_i, \mathbf{d}_j) = \sum_{c \in \mathcal{H} \cup \mathcal{T}} \mathbf{d}_i(c)(1-\mathbf{d}_j(c)),
\end{equation}
where $\mathcal{H}$ and $\mathcal{T}$ denote the sets of head and tail classes, respectively. Based on this, we construct a complementarity matrix $\mathbf{C}$ and apply K-means clustering to form heterogeneous groups. Each cluster achieves a more diverse class distribution across clients, allowing aggregation of general prompts to synthesize diverse features without bias toward heads. This mitigates inter-client gradient disagreement, preventing tail-class signals from being dominated.

The two clustering strategies ensure that class-aware and general prompts are each aggregated within groups that best serve their respective learning objectives.

\noindent \textbf{Bandit-Guided Cluster Synchronization.}
While dual-prompt decoupling and heterogeneity-aware clustering effectively isolate tail gradients locally and within clusters, naive and frequent global aggregation risks re-amplifying head-class influences through untargeted mixing while incurring unnecessary communication overhead. We therefore model the decision of balancing intra-cluster learning and global aggregation as a multi-armed bandit (MAB) problem integrated into CAPT.

In each global round $t$, the scheduler selects an action $(h_s,h_d)$ from the candidate budgets $\mathcal{H}_s \times \mathcal{H}_d$, where $h_s$ and $h_d$ denote the numbers of local iterations for the similarity and heterogeneity phases, respectively. After intra-cluster aggregation, the reward is:
\begin{equation}
\label{eq:reward}
r_t \;=\; \Delta\mathrm{F1}_{\mathrm{tail},t} \;+\; \Delta\mathrm{Sim}_{t},
\end{equation}
where $\Delta\mathrm{F1}_{\mathrm{tail},t}$ is the tail-class F1 improvement and $\Delta\mathrm{Sim}_t$ measures the change in update consistency, defined as the average intra-cluster pairwise cosine similarity of client prompt-update vectors across both clustering phases. To balance exploration and exploitation, we adopt a standard UCB policy~\cite{auer2002finite} to select the next action based on the exponential moving average of past rewards $Q_t(h_s,h_d)$ plus an optimism term scaled by selection counts $n_t(h_s,h_d)$. This adaptive scheduling reduces the number of global aggregations in practice and mitigates over-mixing of head-biased clusters, thereby reinforcing the benefits of our clustering strategies.

\begin{table*}[t]
\centering
\caption{Performance comparison of CAPT against baseline methods on long-tailed datasets under different degrees of data heterogeneity ($\alpha$). Results show accuracy (\%) for overall, head, mid, and tail classes. \blueup{}means improved accuracy compared with PromptFL~\cite{PromptFL}. The best results are in \textbf{bold}. The second-best results are \underline{underlined}.}
\vspace{-0.5em}
\small{
\setlength\tabcolsep{5.pt}
\renewcommand\arraystretch{1.1}
\begin{tabular}{cr||cccc|cccc|cccc}
\hline\thickhline
\rowcolor{gray!20}
\multirow{2}{*}{\textbf{Dataset}} & \multirow{2}{*}{\textbf{Method}} &
\multicolumn{4}{c|}{\textbf{$\alpha = 0.05$}} &
\multicolumn{4}{c|}{\textbf{$\alpha = 0.1$}} &
\multicolumn{4}{c}{\textbf{$\alpha = 0.5$}} \\
\cline{3-14}
\rowcolor{gray!20}
\multirow{-2}{*}{\textbf{Dataset}} & \multirow{-2}{*}{\textbf{Method}} 
  & \textbf{Overall} & \textbf{Head} & \textbf{Mid} & \textbf{Tail}
  & \textbf{Overall} & \textbf{Head} & \textbf{Mid} & \textbf{Tail}
  & \textbf{Overall} & \textbf{Head} & \textbf{Mid} & \textbf{Tail} \\
\hline\hline
\multirow{7}{*}{CIFAR-10-LT}
 & FedCLIP~\cite{fedclip} & 84.81 & \textbf{98.10} & \underline{90.20} & 76.26 & 85.55 & \underline{98.45} & \textbf{90.67} & 77.32 & 86.26 & 97.25 & 89.50 & 79.92 \\
 & \cellcolor{gray!10}PromptFL~\cite{PromptFL} & \cellcolor{gray!10}88.67 & \cellcolor{gray!10}95.30 & \cellcolor{gray!10}86.50 & \cellcolor{gray!10}\underline{84.93} & \cellcolor{gray!10}\underline{89.32} & \cellcolor{gray!10}94.07 & \cellcolor{gray!10}88.83 & \cellcolor{gray!10}85.23 & \cellcolor{gray!10}\underline{90.59} & \cellcolor{gray!10}96.63 & \cellcolor{gray!10}87.43 & \cellcolor{gray!10}\underline{88.77} \\
 & FedTPG~\cite{fedtpg} & 87.78 & 96.70 & 87.90 & 84.14 & 88.91 & 96.60 & 84.70 & \underline{88.36} & 88.90 & \textbf{98.30} & \underline{91.40} & 83.64 \\
 & \cellcolor{gray!10}FedOTP~\cite{fedotp} & \cellcolor{gray!10}83.27 & \cellcolor{gray!10}83.70 & \cellcolor{gray!10}82.13 & \cellcolor{gray!10}83.78 & \cellcolor{gray!10}85.15 & \cellcolor{gray!10}96.65 & \cellcolor{gray!10}85.97 & \cellcolor{gray!10}80.06 & \cellcolor{gray!10}86.97 & \cellcolor{gray!10}97.65 & \cellcolor{gray!10}86.50 & \cellcolor{gray!10}82.98 \\
 & FedPGP~\cite{fedpgp} & \underline{88.76} & 97.05 & 88.10 & 84.64 & 86.96 & 98.15 & 84.83 & 83.76 & 87.49 & \underline{97.95} & 91.23 & 81.06 \\
 & \cellcolor[HTML]{D7F6FF}\textbf{CAPT (Ours)} & \cellcolor[HTML]{D7F6FF}\textbf{93.78} & \cellcolor[HTML]{D7F6FF}\underline{97.50} & \cellcolor[HTML]{D7F6FF}\textbf{90.47} & \cellcolor[HTML]{D7F6FF}\textbf{94.28} & \cellcolor[HTML]{D7F6FF}\textbf{93.98} & \cellcolor[HTML]{D7F6FF}\textbf{98.60} & \cellcolor[HTML]{D7F6FF}\underline{89.70} & \cellcolor[HTML]{D7F6FF}\textbf{94.70} & \cellcolor[HTML]{D7F6FF}\textbf{94.04} & \cellcolor[HTML]{D7F6FF}97.80 & \cellcolor[HTML]{D7F6FF}\textbf{93.17} & \cellcolor[HTML]{D7F6FF}\textbf{93.06} \\
 & & \blueup{5.11} & \blueup{2.20} & \blueup{3.97} & \blueup{9.35} & \blueup{4.66} & \blueup{4.53} & \blueup{0.87} & \blueup{9.47} & \blueup{3.45} & \blueup{1.17} & \blueup{5.74} & \blueup{4.29} \\
\hline\hline
\multirow{7}{*}{CIFAR-100-LT}
 & FedCLIP~\cite{fedclip} & 61.52 & 74.83 & 64.19 & 49.69 & 61.46 & 74.00 & 63.00 & 50.83 & 62.10 & 74.93 & 64.16 & 50.09 \\
 & \cellcolor{gray!10}PromptFL~\cite{PromptFL} & \cellcolor{gray!10}63.21 & \cellcolor{gray!10}74.00 & \cellcolor{gray!10}\underline{68.10} & \cellcolor{gray!10}51.60 & \cellcolor{gray!10}65.09 & \cellcolor{gray!10}\underline{76.34} & \cellcolor{gray!10}\underline{67.16} & \cellcolor{gray!10}55.33 & \cellcolor{gray!10}65.79 & \cellcolor{gray!10}76.03 & \cellcolor{gray!10}63.35 & \cellcolor{gray!10}60.25 \\
 & FedTPG~\cite{fedtpg} & \underline{66.01} & 67.48 & 66.84 & \underline{64.30} & 64.52 & 71.17 & 65.55 & 58.90 & 65.56 & \textbf{79.48} & \underline{69.26} & 52.60 \\
 & \cellcolor{gray!10}FedOTP~\cite{fedotp} & \cellcolor{gray!10}60.74 & \cellcolor{gray!10}70.45 & \cellcolor{gray!10}54.58 & \cellcolor{gray!10}57.51 & \cellcolor{gray!10}61.22 & \cellcolor{gray!10}72.24 & \cellcolor{gray!10}61.74 & \cellcolor{gray!10}52.62 & \cellcolor{gray!10}59.43 & \cellcolor{gray!10}75.41 & \cellcolor{gray!10}61.68 & \cellcolor{gray!10}45.84 \\
 & FedPGP~\cite{fedpgp} & 65.07 & \textbf{78.72} & 64.29 & 55.77 & \underline{65.47} & 76.31 & 62.77 & \underline{59.70} & \underline{66.42} & 71.17 & 65.39 & \underline{63.32} \\
 & \cellcolor[HTML]{D7F6FF}\textbf{CAPT (Ours)} & \cellcolor[HTML]{D7F6FF}\textbf{72.91} & \cellcolor[HTML]{D7F6FF}\underline{78.59} & \cellcolor[HTML]{D7F6FF}\textbf{72.52} & \cellcolor[HTML]{D7F6FF}\textbf{69.10} & \cellcolor[HTML]{D7F6FF}\textbf{73.89} & \cellcolor[HTML]{D7F6FF}\textbf{78.52} & \cellcolor[HTML]{D7F6FF}\textbf{73.97} & \cellcolor[HTML]{D7F6FF}\textbf{70.47} & \cellcolor[HTML]{D7F6FF}\textbf{73.77} & \cellcolor[HTML]{D7F6FF}\underline{79.00} & \cellcolor[HTML]{D7F6FF}\textbf{72.52} & \cellcolor[HTML]{D7F6FF}\textbf{70.95} \\
 & & \blueup{9.70} & \blueup{4.59} & \blueup{4.42} & \blueup{17.50} & \blueup{8.80} & \blueup{2.18} & \blueup{6.81} & \blueup{15.14} & \blueup{7.98} & \blueup{2.97} & \blueup{9.17} & \blueup{10.70} \\
\end{tabular}}
\label{tab:mainResult_baseline}
\end{table*}

\begin{table*}[t]
\caption{Comparison of CAPT with prompt-tuning methods. Results show accuracy for overall, head, mid, and tail classes on CIFAR-10-LT and CIFAR-100-LT, under varying degrees of data heterogeneity. Baseline methods are combined with FedAvg for parameter aggregation. \textbf{Bold} indicates best performance, \underline{underline} indicates second-best performance.}
\vspace{-0.5em}
\centering
\small{
\setlength\tabcolsep{4.4pt}
\renewcommand\arraystretch{1.1}
\begin{tabular}{cr||cccc|cccc|cccc}
\hline\thickhline
\rowcolor{gray!20}
\multirow{2}{*}{\textbf{Dataset}} & \multirow{2}{*}{\textbf{Method}} &
\multicolumn{4}{c|}{\textbf{$\alpha = 0.05$}} &
\multicolumn{4}{c|}{\textbf{$\alpha = 0.1$}} &
\multicolumn{4}{c}{\textbf{$\alpha = 0.5$}} \\
\cline{3-14}
\rowcolor{gray!20}
\multirow{-2}{*}{\textbf{Dataset}} & \multirow{-2}{*}{\textbf{Method}}  & \textbf{Overall} & \textbf{Head} & \textbf{Mid} & \textbf{Tail}
  & \textbf{Overall} & \textbf{Head} & \textbf{Mid} & \textbf{Tail}
  & \textbf{Overall} & \textbf{Head} & \textbf{Mid} & \textbf{Tail} \\
\hline\hline
\multirow{5}{*}{CIFAR-10-LT}
 & CoCoop+FedAvg~\cite{cocoop} & 87.77 & \underline{97.70} & \textbf{90.57} & 82.12 & 87.73 & 95.85 & 87.97 & 84.34 & 88.21 & 97.05 & 85.80 & 86.12 \\
 & \cellcolor{gray!10}MaPle+FedAvg~\cite{maple} & \cellcolor{gray!10}\underline{90.39} 	& \cellcolor{gray!10}\textbf{99.35} 	& \cellcolor{gray!10}89.23 	& \cellcolor{gray!10}87.50 	& \cellcolor{gray!10}\underline{91.25} 	& \cellcolor{gray!10}\textbf{98.90} 	& \cellcolor{gray!10}88.40 	& \cellcolor{gray!10}\underline{89.90} 	& \cellcolor{gray!10}\underline{92.05} 	& \cellcolor{gray!10}\textbf{99.20} 	& \cellcolor{gray!10}\textbf{94.07} 	& \cellcolor{gray!10}87.98 \\
 & KgCoOp+FedAvg~\cite{KgCoOp} & 88.21 	& 97.05 	& 87.67 	& 85.00 	& 87.67 	& 97.50 	& 85.40 	& 85.10 	& 86.93 	& \underline{97.95} 	& 89.20 	& 81.16 \\
 & \cellcolor{gray!10}CLIP-LoRA+FedAvg~\cite{clip-lora} & \cellcolor{gray!10}87.77 	& \cellcolor{gray!10}89.65 	& \cellcolor{gray!10}86.10 	& \cellcolor{gray!10}\underline{88.02} 	& \cellcolor{gray!10}87.90 	& \cellcolor{gray!10}89.31 	& \cellcolor{gray!10}\underline{89.31} 	& \cellcolor{gray!10}88.26 	& \cellcolor{gray!10}87.91 	& \cellcolor{gray!10}89.81 	& \cellcolor{gray!10}86.34 	& \cellcolor{gray!10}\underline{88.09} \\
 & \cellcolor[HTML]{D7F6FF}\textbf{CAPT (Ours)} & \cellcolor[HTML]{D7F6FF}\textbf{93.78} & \cellcolor[HTML]{D7F6FF}97.50 & \cellcolor[HTML]{D7F6FF}\underline{90.47} & \cellcolor[HTML]{D7F6FF}\textbf{94.28} & \cellcolor[HTML]{D7F6FF}\textbf{93.98} & \cellcolor[HTML]{D7F6FF}\underline{98.60} & \cellcolor[HTML]{D7F6FF}\textbf{89.70} & \cellcolor[HTML]{D7F6FF}\textbf{94.70} & \cellcolor[HTML]{D7F6FF}\textbf{94.04} & \cellcolor[HTML]{D7F6FF}97.80 & \cellcolor[HTML]{D7F6FF}\underline{93.17} & \cellcolor[HTML]{D7F6FF}\textbf{93.06} \\
 \hline\hline
\multirow{5}{*}{CIFAR-100-LT} 
 & CoCoop+FedAvg~\cite{cocoop} & 59.37 	& \underline{80.14} 	& 62.16 	& 42.15 	& 60.85 	& 77.76 	& 64.94 	& 47.35 	& 61.62 	& 77.76 	& 64.94 	& 47.35 \\
 & \cellcolor{gray!10}MaPle+FedAvg~\cite{maple} & \cellcolor{gray!10}\underline{70.12} 	& \cellcolor{gray!10}\textbf{82.17} 	& \cellcolor{gray!10}\underline{72.45} 	& \cellcolor{gray!10}59.58 	& \cellcolor{gray!10}\underline{70.50} 	& \cellcolor{gray!10}\textbf{82.55} 	& \cellcolor{gray!10}\textbf{74.23} 	& \cellcolor{gray!10}58.88 	& \cellcolor{gray!10}\underline{70.94} 	& \cellcolor{gray!10}\textbf{85.55} 	& \cellcolor{gray!10}\textbf{74.45} 	& \cellcolor{gray!10}57.62 \\
 & KgCoOp+FedAvg~\cite{KgCoOp} & 65.81 	& 79.00 	& 67.48 	& 54.95 	& 66.75 	& \underline{80.83} 	& 70.23 	& 53.85 	& 66.99 	& \underline{80.86} 	& 71.32 	& 53.58 \\
 & \cellcolor{gray!10}CLIP-LoRA+FedAvg~\cite{clip-lora} & \cellcolor{gray!10}64.97 	& \cellcolor{gray!10}66.76 	& \cellcolor{gray!10}65.51 	& \cellcolor{gray!10}\underline{63.25} 	& \cellcolor{gray!10}64.88 	& \cellcolor{gray!10}66.74 	& \cellcolor{gray!10}66.74 	& \cellcolor{gray!10}\underline{63.15} 	& \cellcolor{gray!10}64.94 	& \cellcolor{gray!10}66.59 	& \cellcolor{gray!10}65.55 	& \cellcolor{gray!10}\underline{63.27} \\
 & \cellcolor[HTML]{D7F6FF}\textbf{CAPT (Ours)} & \cellcolor[HTML]{D7F6FF}\textbf{72.91} & \cellcolor[HTML]{D7F6FF}78.59 & \cellcolor[HTML]{D7F6FF}\textbf{72.52} & \cellcolor[HTML]{D7F6FF}\textbf{69.10} & \cellcolor[HTML]{D7F6FF}\textbf{73.89} & \cellcolor[HTML]{D7F6FF}78.52 & \cellcolor[HTML]{D7F6FF}\underline{73.97} & \cellcolor[HTML]{D7F6FF}\textbf{70.47} & \cellcolor[HTML]{D7F6FF}\textbf{73.77} & \cellcolor[HTML]{D7F6FF}79.00 & \cellcolor[HTML]{D7F6FF}\underline{72.52} & \cellcolor[HTML]{D7F6FF}\textbf{70.95} \\
\end{tabular}}
\captionsetup{font=small}
\label{tab:mainResult_prompt-tuning}
\vspace{-5pt}
\end{table*}

\subsection{Privacy Discussion}
CAPT estimates each client's label distribution entirely on-device using per-class gradient-magnitude proxies (see Eq.~\eqref{eq:local_proxy}), so this client-side step does not transmit raw data nor per-sample signals to the server. In the server aggregation, we adopt anonymous aggregated statistics without leaking sample-level information, which is a common practice in Fed-LT methods such as Fed-GraB~\cite{Fed-Grab}, FedLoGe~\cite{Fedloge}, and RUCR~\cite{RUCR}. Additionally, as demonstrated by PromptFL~\cite{PromptFL}, prompt-based gradients inherently offer stronger privacy protection than full-model gradients: inversion attacks on prompt gradients produce only noisy artifacts rather than meaningful data reconstructions, a notable advantage compared to traditional federated methods~\cite{gradientFL}.

\section{Experiment}

\subsection{Experimental Setup}
\noindent \textbf{Datasets.} We conduct experiments on five long-tailed datasets, including
CIFAR-10-LT, CIFAR-100-LT, and ImageNet-LT, Places-LT datasets. For the first three datasets, we follow \cite{cao2019learning} to create long-tailed versions using exponential decay with imbalance factors ($\rho \!= \! N_{\max} / N_{\min}$). We use imbalance factors of 100 for CIFAR-10-LT, CIFAR-100-LT by default. For ImageNet-LT and Places-LT, we use the standard version from \cite{imagenet-lt}, with class sizes ranging from 5 to 1,280 samples. All datasets maintain their original test sets. Following \cite{Fedloge}, we analyze performance across Head, Mid, and Tail class groups based on sample frequency.

\noindent \textbf{Implementation Details.}
We evaluate CAPT on ResNet-50 and ViT-B/16 backbones. In our federated learning setup, we simulate a system with 20 clients, where the local datasets are partitioned using a symmetric Dirichlet distribution with a concentration parameter $\alpha \in \{0.05, 0.1, 0.5\}$ to control the degree of non-IID data. For local training, we use the SGD optimizer with a constant learning rate of 1e-3. The batch size is set to 32 for small-scale datasets, and 256 for large-scale datasets. We perform 100 rounds of communication training, with 1 local epoch per round for CLIP-based methods and 5 local epochs per round in other cases. In each round, 40\% of the clients are randomly selected to participate in the training. The balancing parameter $\lambda$ for the loss function is set to 1. 

\noindent \textbf{Reproducibility.}
CAPT is implemented in PyTorch. All experiments were conducted on NVIDIA GeForce RTX 3090 GPUs running Linux. By default, we set the number of learnable context tokens to 1 for the general prompt and to 3 for the class-aware prompt.

\subsection{Comparison with State-of-the-art Methods}
\noindent \textbf{Comparison with Federated Learning Baselines.}
We compare CAPT with three categories of methods: CLIP-based federated learning approaches, federated versions of centralized CLIP fine-tuning methods. Tab.~\ref{tab:mainResult_baseline} compares CAPT with state-of-the-art CLIP-based federated learning methods across varying heterogeneity levels. CAPT consistently outperforms baselines in overall accuracy and tail-class effectiveness, with tail gains of 4.2--10.8\% over the second-best method, substantially narrowing the head-tail gap by up to 15\% compared to baselines. Fig.~\ref{fig:acc_gains} illustrates these balanced improvements across head-to-tail classes, contrasting with PromptFL's notable tail degradation that exacerbates the head-tail gap.

\begin{table}[t]
\centering
\caption{{Performance comparison of state-of-the-art federated long-tailed learning methods using ResNet backbone.
}}
\small{
\renewcommand\arraystretch{1.1}
\begin{tabular}{r||c|c|c}
\hline\thickhline
\rowcolor{gray!20}
\textbf{Method}
& \textbf{CIFAR-10-LT}
& \textbf{CIFAR-100-LT}
& \textbf{ImageNet-LT}
\\
\hline\hline
\multicolumn{4}{l}{\textcolor{gray!60}{\textit{Training from Scratch}}} \\
FedAvg~\cite{FedAvg}
& 58.22 & 30.47 & 23.85
\\
\rowcolor{gray!10} CReFF~\cite{Creff}
& 70.31 & 32.90 & 25.98
\\
Fed-Grab~\cite{Fed-Grab}
& 70.51 & 36.53 & 32.94
\\
\rowcolor{gray!10} FedLoge~\cite{Fedloge}
& 70.54 & 42.72 & 35.62
\\
RUCR~\cite{RUCR}
& 71.40 & 35.58 & 24.02
\\
\hline
\multicolumn{4}{l}{\textcolor{gray!60}{\textit{Training with Foundation Model}}} \\
\rowcolor{gray!10} PromptFL~\cite{PromptFL}
& 62.68 & 45.77 & 58.69
\\
\hline
\rowcolor[HTML]{D7F6FF}
\textbf{CAPT (Ours)}
& \textbf{81.28}\blueupsmall{9.88} & \textbf{53.83}\blueupsmall{8.06} & \textbf{64.19}\blueupsmall{5.50}
\end{tabular}}
\captionsetup{font=small}
\label{tab:resnet_comparison}
\end{table}

\noindent \textbf{Comparison with Prompt-based Methods.}
To investigate the efficacy of simply federating centralized CLIP fine-tuning methods for Fed-LT, we compare CAPT with federated versions of CoCoOp~\cite{cocoop}, MaPLe~\cite{maple}, KgCoOp~\cite{KgCoOp}, and CLIP-LoRA~\cite{clip-lora} in Tab.~\ref{tab:mainResult_prompt-tuning}. Our results indicate that CAPT consistently outperforms these methods, reducing the head-tail gap and underscoring the significance of CAPT's class-aware prompt design and heterogeneity-aware client clustering in tackling the unique challenges of Fed-LT.

\begin{figure}[!t]
  \centering
  \includegraphics[width=0.45\textwidth]{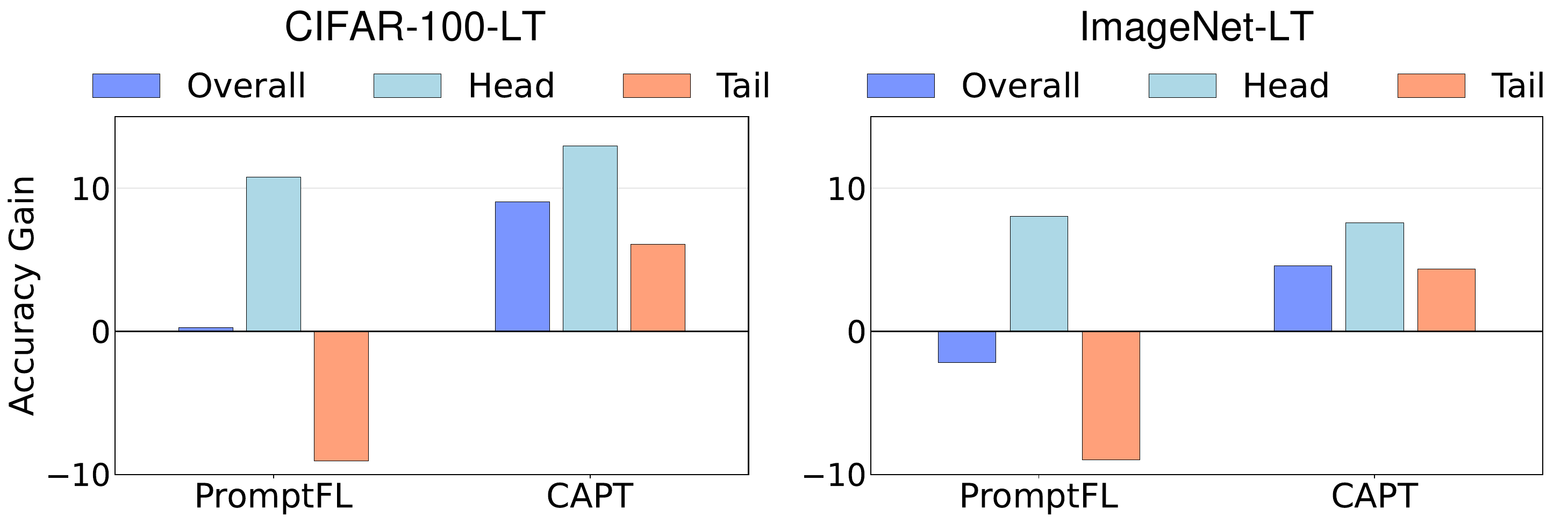}
  \caption{Accuracy gains (\%) of PromptFL (left) and CAPT (right) relative to CLIP on CIFAR-100-LT and ImageNet-LT.}
  \label{fig:acc_gains}
\end{figure}

\begin{figure}[!t]
  \centering
  \includegraphics[width=0.48\textwidth]{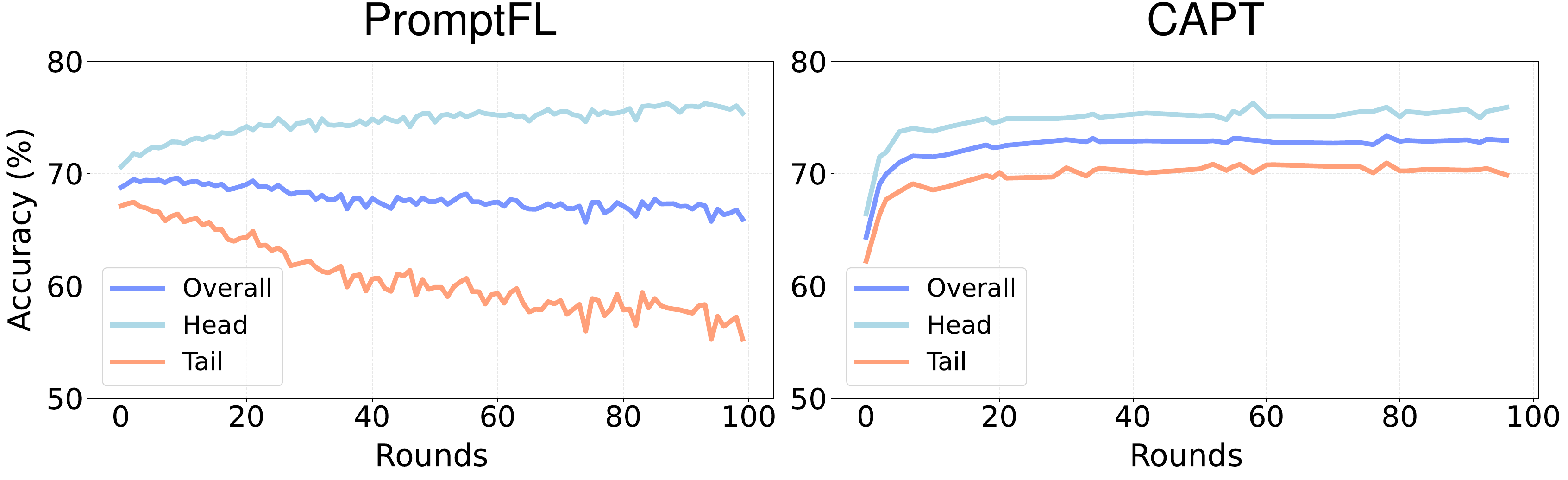}
  \caption{Training dynamics of PromptFL (left) and CAPT (right) on the ImageNet-LT dataset. Each panel plots accuracy (\%) for overall, head-class, and tail-class performance over 100 communication rounds.}
  \label{fig:acc_trend}
\end{figure}

\noindent \textbf{Large-Scale Dataset Comparison.}
We evaluate CAPT on the large-scale ImageNet-LT and Places-LT datasets under varying heterogeneity levels, as reported in Tab.~\ref{tab:imagenet-lt-places-lt-results}. CAPT demonstrates strong generalization from scarce and heterogeneous data, achieving substantial improvements in tail-class accuracy (up to 12.7\% on Places-LT and 11.4\% on ImageNet-LT relative to the strongest baselines). Overall accuracy gains reach 8.8\% on Places-LT and 6.9\% on ImageNet-LT, effectively narrowing the head-tail performance gap by 9--12\% compared to competitors like FedTPG~\cite{fedtpg} and KgCoOp+FedAvg~\cite{KgCoOp} These results highlight CAPT's robustness to client heterogeneity, enabling reliable tail-class representation even in data-sparse regimes.

\begin{table}[t]
\centering
\caption{Performance comparison on the large-scale datasets.}
\small{
\setlength\tabcolsep{1.5pt}
\renewcommand\arraystretch{1.0}
\begin{tabular}{r||cccc|cccc}
\hline\thickhline
\rowcolor{gray!20}
 &
\multicolumn{4}{c|}{\textbf{ImageNet-LT}} &
\multicolumn{4}{c}{\textbf{Places-LT}}
\\
\cline{2-9}
\rowcolor{gray!20}
\multirow{-2}{*}{\textbf{Method}}
& \textbf{All} & \textbf{Head} & \textbf{Mid} & \textbf{Tail} & \textbf{All} & \textbf{Head} & \textbf{Mid} & \textbf{Tail}
\\
\hline\hline
FedCLIP
& 65.4 & 70.1 & 64.7 & 61.9 & 32.3 & 37.7 & 31.6 & 33.1
\\
\rowcolor{gray!10} PromptFL
& 67.1 & 76.0 & 65.6 & 59.8 & 36.4 & 55.0 & 32.2 & 30.3
\\
FedTPG
& 68.2 & 73.4 & 66.7 & 63.5 & 38.1 & \underline{55.1} & 34.0 & 32.9
\\
\rowcolor{gray!10} FedOTP
& 55.5 & 60.6 & 52.2 & 51.5 & 33.0 & 45.9 & 27.6 & 20.7
\\
FedPGP
& 68.3 & 76.6 & 65.9 & 60.0 & 37.5 & 47.0 & 30.0 & 28.5
\\
\rowcolor{gray!10} CoCoop+FedAvg
& 65.9 & 75.2 & 63.8 & 55.6 & 37.4 & 52.8 & 34.4 & 30.7
\\
MaPle+FedAvg
& 67.6 & \underline{78.0} & 65.2 & 57.9 & 37.5 & \textbf{55.6} & 34.1 & 30.3
\\
\rowcolor{gray!10} KgCoOp+FedAvg
& \underline{68.9} & 74.8 & 66.5 & \underline{64.3} & 38.9 & 50.1 & 33.9 & 32.7
\\
CLIP-LoRA+FedAvg
& 68.7 & \textbf{81.0} & \underline{67.9} & 56.1 & \underline{39.84} & 49.8 & \underline{36.0} & \underline{34.3}
\\
\hline
\rowcolor[HTML]{D7F6FF}
\textbf{CAPT (Ours)}
& \textbf{74.0} & 77.0 & \textbf{73.0} & \textbf{71.2}& \textbf{45.2} & 51.6 & \textbf{44.2} & \textbf{43.0}\\
& \blueup{6.9} & \blueup{1.0} & \blueup{7.4} & \blueup{11.4} & \blueup{8.8} & \reddown{3.4} & \blueup{12.0} & \blueup{12.7} \\
\end{tabular}}
\captionsetup{font=small}
\label{tab:imagenet-lt-places-lt-results}
\end{table}

\begin{table}[t]\footnotesize
\centering
\caption{Ablation study of CAPT components, comparing different prompt types and cluster aggregation strategies.}
\small{
\resizebox{\linewidth}{!}{
\setlength\tabcolsep{4.pt}
\renewcommand\arraystretch{1.1}
\begin{tabular}{c||c|cccc}
\hline\thickhline
\rowcolor{gray!20}
\multirow{2}{*}{\textbf{Cluster}}
& \multirow{2}{*}{\textbf{Prompt Type}}
& \multicolumn{4}{c}{\textbf{CIFAR-10-LT}} \\
\cline{3-6}
\rowcolor{gray!20}
\multirow{-2}{*}{\textbf{Cluster}} &\multirow{-2}{*}{\textbf{Prompt Type}} & \textbf{Overall} & \textbf{Head} & \textbf{Mid} & \textbf{Tail} \\
\hline\hline
\multirow{3}{*}{\rotatebox[origin=c]{90}{\textbf{w/o Clu.}}}
& General Prompt        & 89.23 & 97.50 & 90.95 & 84.93 \\
& Class-aware Prompt    & \underline{90.27} & \underline{98.75} & \underline{91.53} & 86.12 \\
& Gen+Cla               & \textbf{90.42} & \textbf{99.30} & 91.27 & \textbf{86.36} \\
\hline

\multirow{3}{*}{\rotatebox[origin=c]{90}{\textbf{w/ Clu.}}}
& General Prompt        &92.00 & 94.55 & 90.57 &91.84 \\
& Class-aware Prompt    & \underline{92.69} & 96.85 & 90.00 & \underline{92.64} \\
& \cellcolor[HTML]{D7F6FF}\textbf{Gen+Cla} & \cellcolor[HTML]{D7F6FF}\textbf{94.04} & \cellcolor[HTML]{D7F6FF}\textbf{97.80} & \cellcolor[HTML]{D7F6FF}\textbf{93.17} & \cellcolor[HTML]{D7F6FF}\textbf{93.06} \\
\end{tabular}}}
\captionsetup{font=small}
\label{tab:dataset_comparison_single_column_revised}
\end{table}

\begin{figure*}[!t]
  \centering
  \includegraphics[width=0.99\textwidth]{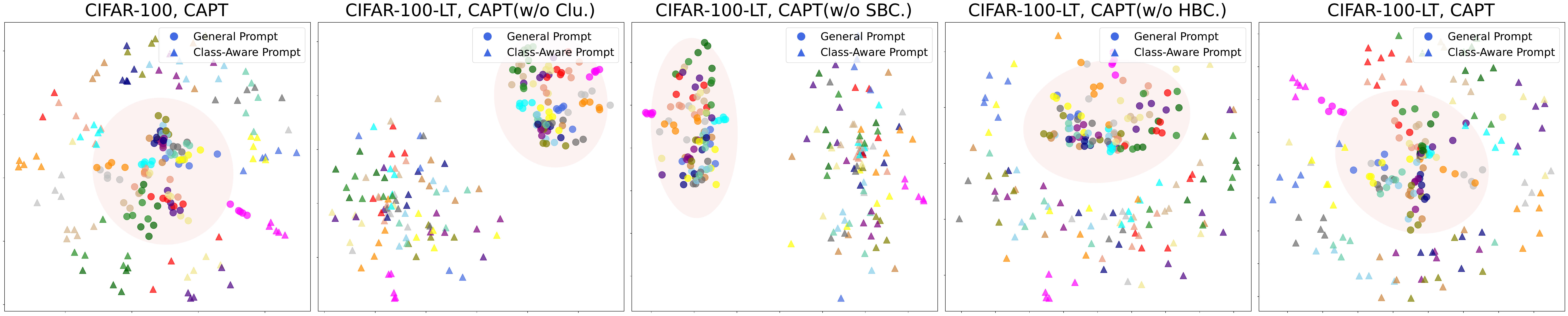}
  \caption{t-SNE visualizations of feature embeddings learned by General Prompts (circles) and Class-Aware Prompts (triangles) under varying configurations. Shaded regions highlight the concentrated area of General Prompt features. Notably, CAPT's embeddings on long-tailed data (rightmost) closely mimic those on balanced data (leftmost), demonstrating robustness to imbalance through effective prompt decoupling and clustering synergy.}
  \label{fig:tsne}
\end{figure*}

\begin{figure}[!t]
  \centering
  \includegraphics[width=0.47\textwidth]{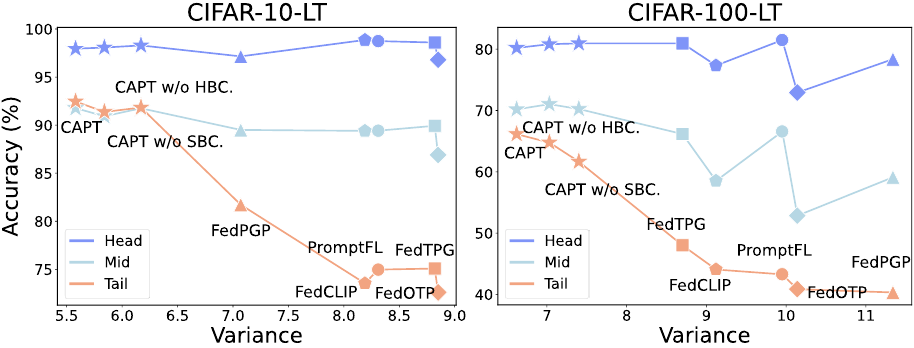}
  \vspace{-0.5em}
  \caption{Accuracy versus inter-client gradient variance on CIFAR-10-LT and CIFAR-100-LT. CAPT variants achieve high accuracy at low variance, while baselines suffer declining tail accuracy as variance increases.}
  \label{fig:variance}
\end{figure}

Training dynamics, visualized in Fig.~\ref{fig:acc_trend}, further reveal CAPT's ability to maintain a stable balance between head- and tail-class accuracies throughout optimization. In contrast, the PromptFL method exhibits a widening gap due to rapid tail-class degradation and even fails to converge, while our method can converge quickly in a short period. This underscores CAPT's effectiveness in preserving tail-class gradients and facilitating balanced global prompt learning.

Additionally, Tab.~\ref{tab:resnet_comparison} compares CAPT against state-of-the-art Fed-LT methods using ResNet backbones, including CReFF~\cite{Creff}, Fed-Grab~\cite{Fed-Grab}, FedLoge~\cite{Fedloge}, and RUCR~\cite{RUCR}. CAPT outperforms all baselines across CIFAR-10-LT, CIFAR-100-LT and ImageNet-LT, with gains of 9--19\% in overall accuracy. Despite these baselines being specialized for addressing class imbalance, CAPT's use of pre-trained foundation models provides a strong initialization, enabling superior handling of long-tailed distributions in federated settings and confirming its versatility across backbones and scales.

\subsection{Ablation Studies}
\label{exp:Ablation Studies}
\noindent \textbf{Impact of Prompt Design and Clustering Strategy.} We assess the effectiveness of CAPT's prompt designs by ablating the general prompt, class-aware prompt, and their combination (Gen+Cla). As shown in Tab.~\ref{tab:dataset_comparison_single_column_revised}, the class-aware prompt outperforms the general prompt across all categories, underscoring the value of class-specific features in long-tailed settings. The Gen+Cla combination yields the best results, leveraging complementary strengths. Adding heterogeneity-aware clustering (w/ Clu.) substantially enhances all configurations, with tail gains of 6-7\%. Notably, clustering narrows performance gaps between prompt variants, indicating it mitigates data heterogeneity, enabling robust tail representations even with simpler prompts.

\noindent \textbf{Influence of Similarity Metrics and Clustering Parameters.}
We assess CAPT's sensitivity to similarity metrics, clustering techniques, and cluster numbers. Results show robustness, with performance varying \textless 2\% across metrics and methods. For cluster counts, the results show stable performance across a moderate range of cluster numbers, highlighting the robustness and practicality of our strategy.

\subsection{Analysis of CAPT}

\noindent \textbf{Observation 1.}
\textit{The dual-clustering strategy enables CAPT to learn prompt embeddings that closely resemble those trained on balanced data.}
We visualize the feature embeddings of the dual prompts with and without our clustering strategy. As shown in Fig.~\ref{fig:tsne}, removing the dual-clustering module (w/o Clu.) causes the two kinds of prompt embeddings to mingle: the general prompt embeddings (circles) drift away from the class-aware prompt embeddings, leading to low alignment, while the class-aware prompt embeddings (triangles) collapse toward the center, producing severe overlap and insufficient inter-class discrimination. In contrast, with the full CAPT, the general prompt embeddings aggregate toward the center, indicating the capture of domain-invariant features, while the class-aware embeddings enlarge the margins between classes; simultaneously, the coarse-grained super-classes of CIFAR-100 emerge as clusters. This observation highlights the dual-clustering strategy's necessity and validates CAPT's efficacy in preserving generalizable representations while enhancing class-specific discriminability.

\vspace{0.3em}
\noindent \textbf{Observation 2.}
\textit{Gradient variance significantly impacts tail classes.} Our motivation and theoretical analysis highlight how inter-client gradient variance exacerbates head-class dominance over tail classes in Fed-LT. We also evaluate it empirically. As shown in Fig.~\ref{fig:variance}, CAPT and its variants exhibit markedly lower inter-client gradient variance, with this variance progressively decreasing as clustering components are added. Data imbalance worsens head-class gradient dominance, impeding convergence, whereas CAPT's decoupling and clustering mitigate this, fostering a low-variance environment for stable updates, as shown in Fig.~\ref{fig:acc_trend}.

\section{Conclusion}
In this paper, we proposed CAPT, a novel framework for federated learning that effectively addresses long-tailed data distributions and client heterogeneity. Our approach introduces a dual-prompt design with class-aware and general prompts, coupled with a heterogeneity-aware client clustering strategy comprising similarity-based and heterogeneity-based clustering. Notably, CAPT significantly outperforms state-of-the-art methods by achieving substantial gains in both overall and tail-class accuracy across various degrees of data heterogeneity, all while maintaining competitive performance on head classes. These results highlight CAPT's effectiveness in tackling the challenges of federated long-tailed learning, paving the way for more robust and equitable machine learning models in decentralized settings.

\begin{acks}
This study was supported in part by the National Natural Science Foundation of China under Grants 62431004, 62376233, and 62372388; in part by the Natural Science Foundation of Fujian Province under Grant 2024J09001. SNG is supported by a CoSeC Fellowship.
\end{acks}

\bibliographystyle{ACM-Reference-Format}
\bibliography{sample-base}

\end{document}